\documentclass{article}

\usepackage[final]{neurips_2021}

\usepackage{microtype}
\usepackage{graphicx}
\usepackage{subfigure}
\usepackage{booktabs}

\usepackage{notations}
\usepackage{mathletters}

\usepackage{hyperref}

\usepackage[utf8]{inputenc} 
\usepackage[T1]{fontenc}    
\hypersetup{
    colorlinks = true,
    allcolors = blue
}
\usepackage{url}            
\usepackage{booktabs}       
\usepackage{amsfonts}       
\usepackage{nicefrac}       
\usepackage{microtype}      
\usepackage{times}
\usepackage[scaled=0.8]{beramono}
\usepackage{fancybox}
\usepackage{changepage}
\usepackage{enumitem}
\usepackage{comment}
\usepackage{notations}
\usepackage{stmaryrd}
\usepackage{bold-extra}
\usepackage{xcolor}

\usepackage{appendix} 

\usepackage{algorithm,algorithmic}

\usepackage{graphicx}
\graphicspath{{.}{figs/}}
\usepackage{sidecap}
\usepackage{amsmath,amsfonts,amssymb, tikz}
\usepackage{mathletters} 
\usetikzlibrary{arrows,chains,matrix,positioning,scopes}

\usepackage{mathtools}
\usepackage{xspace}

\newcommand{\beq}{\begin{equation}}
\newcommand{\eeq}{\end{equation}}
\newcommand{\beqa}{\begin{eqnarray}}
\newcommand{\eeqa}{\end{eqnarray}}
\newcommand{\beqan}{\begin{eqnarray*}}
	\newcommand{\eeqan}{\end{eqnarray*}}

\newcommand{\beqannumb}{\begin{eqnarray}}
\newcommand{\eeqannumb}{\end{eqnarray}}

\usepackage[colorinlistoftodos, textwidth=4cm, shadow]{todonotes}

\long\def\acks#1{\vskip 0.3in\noindent{\large\bf Acknowledgments}\vskip 0.2in
\noindent #1}
\newcommand{\BlackBox}{\rule{1.5ex}{1.5ex}}  
\newenvironment{proof}{\par\noindent{\bf Proof\ }}{\hfill\BlackBox\\}
 
\newtheorem{theorem}{Theorem}
\newtheorem{lemma}[theorem]{Lemma} 
\newtheorem{proposition}[theorem]{Proposition} 
\newtheorem{remark}[theorem]{Remark}
\newtheorem{corollary}[theorem]{Corollary}
\newtheorem{definition}[theorem]{Definition}

\usepackage[utf8]{inputenc} 
\usepackage[T1]{fontenc}    
\usepackage{hyperref}       
\usepackage{url}            
\usepackage{booktabs}       
\usepackage{amsfonts}       
\usepackage{nicefrac}       
\usepackage{microtype}      
\usepackage{xcolor}         

\title{Indexed Minimum Empirical Divergence for Unimodal Bandits}

\author{
  Hassan Saber
    \\
  Université de Lille, Inria, CNRS, Centrale Lille
  \\
  UMR 9189 – CRIStAL, F-59000 Lille, France \\
  \texttt{hassan.saber@inria.fr} \\
  \And
  Pierre M{\'e}nard
  \\
  Otto von Guericke Universit{\"a}t Magdeburg\\
  \texttt{pierre.menard@ovgu.de}\\
  
  \And 
  
  Odalric-Ambrym Maillard
    \\
  Université de Lille, Inria, CNRS, Centrale Lille
  \\
  UMR 9189 – CRIStAL, F-59000 Lille, France \\
  \texttt{odalric.maillard@inria.fr} \\

}

\begin{document}

\maketitle

\begin{abstract}
  We consider a multi-armed bandit problem specified by a set of one-dimensional family exponential distributions endowed with a unimodal structure. 
	    We introduce \IMEDUB, an algorithm that optimally exploits the unimodal-structure, by adapting to this setting the Indexed Minimum Empirical Divergence (\IMED) algorithm introduced by \cite{honda2015imed}.  Owing to our proof technique, we are able to provide a concise finite-time analysis of the \IMEDUB algorithm. Numerical experiments show that \IMEDUB competes with the state-of-the-art algorithms.
\end{abstract}

	\section{Introduction} \label{sec:intro}
	
	The multi-armed bandit problem is a popular framework to formalize sequential decision making problems.
	It was first introduced in the context of medical trials \citep{thompson1933likelihood,thompson1935criterion} and later formalized by \cite{ro52}:
	A bandit is specified by a configuration, that is a set of unknown probability distributions, $\nu  \!=\! (\nu_a)_{a\in\cA}$ with means $(\mu_{a})_{a\in\cA}$.
	At each time $ t \!\in\!\Nat $, the learner chooses an arm $ a_t \!\in\! \cA $, based only on the past,
	the learner then receives and observes a reward $ X_t $, conditionally independent, sampled according to $ \nu_{a_t} $. The goal of the learner is  to maximize the expected sum of rewards received over time (up to some unknown horizon $T$), or 
	equivalently minimize the \textit{regret} with respect to the algorithm constantly receiving the highest mean reward 
	$$  R(\nu,T) = \Esp_\nu\!\brackets{\sum_{t=1}^T \mu^\star - X_t}  \text{ where }  \mu^\star=\max_{a \in \cA}\mu_{a}\,. $$ 
	Both means and distributions are \textit{unknown}, which makes the problem non trivial, and the learner only knows that $\nu\!\in\!\cD$  where $\cD$ is a given set of bandit configurations.
	This problem received increased attention in the middle of the $20^{\text{th}}$ century, and the seminal paper \cite{lai1985asymptotically} established the first  lower bound on the cumulative regret, showing that designing an algorithm 
	that is optimal uniformly over a given set of configurations $\cD$ comes with a price. 
	The study of the lower performance bounds in multi-armed bandits successfully lead to the development of asymptotically optimal algorithms for specific configuration sets, such as the \KLUCB algorithm \citep{lai1987adaptive,CaGaMaMuSt2013,maillard2018boundary} for exponential families, or alternatively the \DMED and \IMED algorithms from \cite{honda2011asymptotically,honda2015imed}.
	The lower bounds from \cite{lai1985asymptotically}, later extended by \cite{burnetas1997optimal} did not cover all possible configurations, and in particular \textit{structured} configuration sets were not  handled until \cite{agrawal1989asymptotically} and then \cite{graves1997asymptotically} established generic lower bounds. Here, structure refers to the fact that pulling an arm may reveals information that enables to refine estimation of other arms. 
	Unfortunately, designing  numerical efficient algorithms that are provably optimal remains a challenge for many structures.

	\paragraph{Structured configurations.}
	Motivated by the growing popularity of bandits in a number of industrial and societal application domains, the study of \textit{structured configuration sets} has received increasing attention over the last few years:
	The linear bandit problem is one typical illustration \cite{abbasi2011improved, srinivas2010gaussian, durand2017streaming}, for which the linear structure considerably modifies the achievable lower bound, see \cite{lattimore2017end}. 
	The study of a \textit{unimodal} structure naturally appears in many contexts, e.g.  single-peak preference economics, voting theory
	or wireless communications, and has been first considered in \cite{yu2011unimodal} from a bandit perspective, then in  \cite{combes2014unimodal} and \cite{kaufmann2020uts} providing an explicit lower bound together with an algorithm exploiting this specific structure.
	Other structures  include Lipschitz bandits \cite{magureanu2014oslb}, and we refer to the manuscript \cite{magureanu2018efficient} for other examples, such as cascading bandits that are useful in the context of recommender systems.
	In \cite{combes2017minimal}, a generic algorithm is introduced called \OSSB (Optimal Structured Stochastic Bandit), stepping the path towards generic multi-armed bandit algorithms that are  adaptive to a given structure. More recently in \cite{degenne20b}, the authors introduce an adaptation of the \KLUCB strategy to handle structured multi-armed bandit problems.
	
	\paragraph{Unimodal-structure.}  We assume a \textit{unimodal} structure similar to that considered in \cite{yu2011unimodal} and  \cite{combes2014unimodal}. That is, there exists an undirected graph $G \!=\! (\cA, E)$ whose vertices are arms $\cA$, and whose edges $E$ characterize a partial order  among means $(\mu_a)_{a\in \cA}$. This partial order is assumed unknown to the learner. We assume that there exists a unique optimal arm $ a^\star\!=\!\argmax_{a \in \cA}\mu_a$ and that for  all sub-optimal arm $a\!\neq\! a^\star$, there exists a path $P_a \!=\! (a_1 \!=\! a, \dots, a_{\ell_a} \!=\! a^\star) \!\in\! \cA^{\ell_a} $ of length $\ell_a \!\geq\! 2$ such that for all $ i\!\in\! [1,\ell_a -1]$, $(a_i, a_{i+1}) \in E$ and $\mu_{a_i} < \mu_{a_{i+1}}$. Lastly, we assume that $\nu \!\subset\!\cP\!\coloneqq\!\Set{p(\mu), \mu\!\in\!\Theta}$, where  $p(\mu)$ is an exponential-family distribution probability with density $f(\cdot, \mu)$ with respect to some positive measure $\lambda$ on $\Real$ and mean $\mu \!\in\!\Theta \!\subset\!\Real$. $\cP$ is assumed to be known to the learner. Thus, for all $a \!\in\! \cA$ we have $\nu_a \!=\! p(\mu_a)$. 
	We denote by $\cD_{(\cP,G)}$ or simply $\cD$ the structured set of such unimodal-bandit distributions characterized by $\left(\cP,G\right)$. In the following, we assume that $\cP$ is a set of one-dimensional exponential family distributions.

	\paragraph{Contributions.} In this paper, we provide novel regret minimization results related to the unimodal structure. We first revisit the Indexed Minimum Empirical Divergence (\IMED) algorithm from \cite{honda2015imed} introduced for unstructured multi-armed bandits, and adapt it to the unimodal-structured setting. We introduce  in Section~\ref{sec:imed_algo} the \IMEDUB algorithm that is limited to the pulling of the current best arm or their no more than $d$ nearest arms at each time step, with $d$ the maximum degree of nodes in $G$.
	Being constructed from \IMED, \IMEDUB does not require any optimization procedure and does not separate exploration from exploitation rounds. \IMEDUB appears to be a \textit{local} algorithm. 
	We prove in Theorem~\ref{th:asymptotic_optimality} that \IMEDUB is asymptotically optimal. Furthermore, this novel algorithm competes with the state-of-the-art algorithms in practice. This is confirmed by numerical illustrations on synthetic data.
	  We believe that the construction of this algorithm together with the proof techniques developed in this paper are of independent interest for the bandit community.

	\paragraph{Notations.} Let $\nu \!\in\! \cD$. Let $\mu^\star \!=\! \max_{a \in \cA }\mu_a $ be the optimal mean  and $a^\star \!=\! \argmax_{a \in \cA}{\mu_a}$ be the optimal arm of $\nu$. We define for an arm $a\!\in\! \cA$ its sub-optimality gap $\Delta_a \!=\! \mu^\star \!-\! \mu_a$.  Considering an horizon $T\!\geq\! 1$, thanks to the chain rule we can rewrite the regret as follows:
	\begin{equation}
	R(\nu,T) = \sum_{a \in \cA} \Delta_a\, \Esp_\nu\big[N_a(T)\big]\,,
	\label{eq:chain_rule}
	\end{equation}
	where $ N_a(t) \!=\! \sum_{s=1}^t \ind_{\Set{a_s = a }} $ is the number of pulls of arm $a$  at time $t$.
	
	\section{Regret lower bound}
	\label{sec:lower_bounds}
	
	In this subsection, we recall for completeness the known lower bound on the regret when we assume a unimodal structure. 	In order to obtain non trivial lower bound we consider
	algorithms that are \textit{consistent} (aka uniformly-good).
	\begin{definition}[Consistent algorithm]\label{def:consistent}
		An algorithm is consistent on $\cD$ if for all configuration $\nu\in \cD$, for all sub-optimal arm $a$, for all $ \alpha >  0$, 
		\[
		\limT\Esp_\nu \!\left[\dfrac{N_a(T)}{T^\alpha}\right] = 0\,. 
		\]
	\end{definition} 
	We can derive from the notion of consistency an asymptotic lower bound on the regret, see \cite{combes2014unimodal}.  
	\begin{proposition}[Lower bounds on the regret]\label{prop:LB_regret}Let us consider a  consistent algorithm. Then, for all configuration $ \nu \!\in\! \cD$, it must be that
		\[
		\liminfT \dfrac{R(\nu,T)}{\log(T)} \geq c(\nu):= \sum_{a \in \cV_{a^\star}} \dfrac{\Delta_a}{\KLof{\mu_a}{\mu^\star}} \,,
		\]
		where  $\KLof{\mu}{\mu'} \!=\!\int_{\Real}\!\log\!\left(f(x,\mu)/f(x,\mu')\right)\!f(x,\mu) \lambda(\mathrm{d}x)$ denotes the Kullback-Leibler divergence between $\nu\!=\!p(\mu)$ and $\nu'\!=\!p(\mu')$, for $\mu,\mu' \!\in\! \Theta$. 
		\label{prop:lower_bound}
	\end{proposition}
	\begin{remark} The quantity $c(\nu)$ is a fully explicit function of $\nu$ (it does not require solving any optimization problem) for some set of distributions $\nu$ (see Remark~\ref{lb Bern}).
	This useful property no longer holds in general for arbitrary structures. Also, it is noticeable that $c(\nu)$ does not involve all the sub-optimal arms but only the ones in $\cV_{a^\star}$. This indicates that sub-optimal arms outside $\cV_{a^\star}$ are sampled $o(\log(T))$, which contrasts with the unstructured stochastic multi-armed bandits. 	See \cite{combes2014unimodal} for further insights.
	\end{remark}
	\begin{remark} \label{lb Bern}  For Bernoulli distributions, a  possible setting is to assume $\lambda = \delta_0 + \delta_1$ (with $\delta_0, \delta_1$ Dirac measures), $\Theta\!=\!(0,1)$ and for $\mu\!\in\!\Theta$, $f(\cdot,\mu)\!=: x \!\in\!\Set{0,1} \mapsto \mu^x(1-\mu)^{1-x}$. Then for all $\mu,\mu'\!\in\!(0,1)$, $\KLof{\mu}{\mu'}\!=\! \mu\log\!\left(\mu/\mu'\right)+(1\!-\!\mu)\log\!\left((1\!-\!\mu)/(1\!-\!\mu')\right)$. 
	For Gaussian distributions (variance $\sigma^2 \!=\!1$), we assume $\lambda$ to be the Lebesgue measure, $\Theta\!=\!\Real$, and for $\mu \!\in\!\Real$, $f(\cdot,\mu)\!= : x \!\in\!\Real \mapsto (\sqrt{2\pi})^{-1}e^{-(x-\mu)^2\!/2}$. Then for all $\mu,\mu'\!\in\!\Real$,  $ \KLof{\mu}{\mu'}\!=\! (\mu' \!-\! \mu)^2\!/2 $. 
	For Exponential distributions, we assume $\lambda$ to be the Lebesgue measure, $\Theta\!= ]0\,;+\infty[$, and for $\mu \!>\!0$, $f(\cdot,\mu)\!= : x \!>\!0 \mapsto e^{-x/\mu}/\mu$. Then for all $\mu,\mu'\!>\!0$,  $ \KLof{\mu}{\mu'}\!=\! \log\!\left(\mu'/\mu\right)\!+\!\mu/\mu'\!-\!1$.
	\end{remark}
	
	\section{Optimal algorithm for unimodal-structured bandits}
	\label{sec:imed_algo}
	
	We present in this section a novel algorithm that  matches the asymptotic lower bound of Proposition~\ref{prop:LB_regret}. This algorithm is inspired by the Indexed Minimum Empirical Divergence (\IMED) proposed by \cite{honda2011asymptotically}. The general idea behind this algorithm is, following the intuition given by the lower bound, to narrow on the current best arm and its neighbourhood for pulling an arm at a given time step.
	\paragraph{Notations.} The empirical mean of the rewards from the arm $a$ is denoted by $ \muhat_a(t) \!=\!\sum_{s=1}^ t{\ind_{\Set{ a_s = a }} X_s}/N_a(t) $ if $ N_a(t)\!>\! 0  $, $ 0 $ otherwise. We also denote by $\muhat^\star(t) \!=\! \max_{a\in\cA}\muhat_a(t)$ and $\Ahat^\star(t) \!=\! \argmax\limits_{a \in \cA}\muhat_a(t)$ respectively the current best mean and the current set of optimal arms.

	\subsection{The \IMEDUB algorithm.}
	We first pull each arm once. For all arm $a \!\in\! \cA$ and time step $t \!\geq\! 1$ we introduce the \IMED index 
	$$ I_a(t) = N_a(t) \, \KLof{\muhat_a(t)}{\muhat^\star(t)}  + \log\!\left(N_a(t)\right) , $$ with the convention $0\!\times\!\infty \!=\! 0$.
	This index can be seen as a transportation cost for moving a sub-optimal arm to an optimal one plus an exploration term: the logarithm of the number of pulls. When an optimal arm is considered, the transportation cost is null and there is only the exploration part. Note that, as stated in \cite{honda2011asymptotically}, $I_{a}(t)$ is an index in the weaker sense since it cannot be determined only by samples from the arm $a$ but also uses the empirical mean of the current optimal arm. We define \IMEDUB (Indexed Minimum Empirical Divergence for Unimodal Bandits), described in Algorithm~\ref{alg:imedub}, to be the algorithm consisting of pulling an arm  $a_t \!\in\! \Set{\ahat^\star_t}\!\cup\!\cV_{\ahat^\star_t}$ with minimum index at each time step $t$, where is $\ahat^\star_t \!\in\! \argmin_{\ahat^\star \in \Ahat^\star(t) }N_{\ahat^\star}(t)$ is a current best arm. This is a natural algorithm since the lower bound on the regret given in Proposition~\ref{prop:LB_regret} involves only the arms in $\cV_{a^\star}$, the neighbourhood of the arm $a^\star$ of maximal mean.
	\begin{algorithm}[H]
		\caption{\IMEDUB}
		\label{alg:imedub}
		\begin{algorithmic}
		    \STATE Pull each arm once
			\FOR{$ t = \abs{\cA} \dots T-1$}
			\STATE Choose $\ahat^\star_t \in \argmin\limits_{\ahat^\star \in \Ahat^\star(t) }N_{\ahat^\star}(t) $ (chosen arbitrarily)
			\STATE  Pull $a_{t+1} \in \argmin\limits_{a \in \Set{\ahat^\star_t}\cup\cV_{\ahat^\star_t}}I_a(t)$ (chosen arbitrarily)
			\ENDFOR
		\end{algorithmic}
	\end{algorithm}

	\subsection{Asymptotic optimality of \IMEDUB}
	
	In this section, we state the main theoretical result of this paper.
	\begin{theorem}[Upper bounds] \label{th:upper bounds}  Let us consider a set of distributions $ \nu \!\in\! \cD$ and let $a^\star$ its optimal arm. Let $\cV_{a^\star}$  be the  sub-optimal arms in the neighbourhood of $a^\star$. Then under the \IMEDUB 
	algorithm for all $0 \!<\! \epsilon \!<\! \epsilon_\nu $, for all horizon time $ T \!\geq\! 1$, for all $a \!\in\!\cV_{a^\star}$,
		\[
		\Esp_\nu[N_{a}(T)] \leq \dfrac{1 + \alpha_\nu(\epsilon)}{\KLof{\mu_a}{\mu_{a^\star}}} \log(T) + 2d\, C_\epsilon  \sqrt{\log(c_\epsilon T)} + d \!\left(1 + c_{\epsilon_\nu}^{-1}\right) + d(2d+3)  \dfrac{ 2\sigma_{\!\epsilon_\nu}^2\e^{\epsilon_\nu^2/2\sigma_{\!\epsilon}^2}}{\epsilon^2}  +  1 
		\]
		and, for all $a \!\notin\! \Set{a^\star}\!\cup\!\cV_{a^\star} $, 
		\[ \Esp_\nu[N_{a}(T)] \leq 2d\, C_\epsilon  \sqrt{\log(c_\epsilon T)}+ d \!\left(1 + c_{\epsilon_\nu}^{-1}\right) + d(2d+3)  \dfrac{ 2\sigma_{\!\epsilon_\nu}^2\e^{\epsilon_\nu^2/2\sigma_{\!\epsilon}^2}}{\epsilon^2} +  1 \,,
		\]
		where  $d$ is the maximum degree of nodes in $G$, $\epsilon_\nu \!=\!  \min_{a \neq a' }\abs{\mu_a \!-\! \mu_{a'}}\!/2$, \newline $\sigma_{\!\epsilon}^2 \!=\!\max\limits_{a\in\cA} \Set{\mathbb{V}_{_{X\sim p(\mu')}}(X) \!: \mu'  \!\in\! [\mu_a \!-\!\epsilon\,, \mu_a]}$ and  $c_\epsilon, C_\epsilon \!>\! 0$ are the constants involved in Theorem~\ref{thm:boundary_crossing}. $\alpha_\nu(\cdot)$ is a non-negative function depending only on $\nu$ such that $\lim\limits_{\epsilon \to 0}\alpha_\nu(\epsilon)\!=\!0$ (see Section~\ref{imed_unimodal notations} for more details).
	\end{theorem}

	In particular one can note that the arms in the neighbourhood of the optimal one are pulled $\cO\!\left(\log(T)\right)$ times while the other sub-optimal arms are pulled $\cO\!\left(\sqrt{\log(T)}\right)$ of times under \IMEDUB. This is coherent with the lower bound that only involves the neighbourhood of the best arm.
	More precisely, combining Theorem~\ref{th:upper bounds} and the chain rule~\eqref{eq:chain_rule} gives the asymptotic optimality of \IMEDUB with respect to the lower bound of Proposition~\ref{prop:LB_regret}.
	\begin{corollary}[Asymptotic optimality]With the same notations as in Theorem~\ref{th:upper bounds}, then under the \IMEDUB algorithm
		\[
		\limsupT \dfrac{R(\nu,T)}{\log(T)} \leq c(\nu) = \sum\limits_{a \in \cV_{a^\star} } \dfrac{\Delta_a}{\KLof{\mu_a}{\mu_{a^\star}}} \,.
		\]
		\label{th:asymptotic_optimality}
	\end{corollary} 
	A finite time analysis of \IMEDUB is provided in following Section~\ref{sec : imed_analysis}.

	\section{\IMEDUB finite time analysis}
	\label{sec : imed_analysis}
	At a high level, the key interesting step of the proof is to realize that the considered algorithm implies empirical lower and empirical upper bounds on the numbers of pulls (see Lemma~\ref{lem:unimodal empirical lower bounds}, Lemma~\ref{unimodal empirical upper bounds}). Then, based on concentration lemmas (see Section~\ref{app: concentration_lemmas}), the algorithm-based empirical lower bounds ensure the reliability of the estimators of interest (Lemma~\ref{lem:reliable_means}). Interestingly, this makes use of arguments based on recent concentration of measure that enable to control the concentration without adding some $\log\log$ bonus (such a bonus was required for example in the initial analysis of the KL-UCB strategy from \cite{CaGaMaMuSt2013}). 
	Then, combining the reliability of these estimators with the obtained algorithm-base empirical upper bounds, we obtain upper bounds on the average numbers of pulls (Theorem~\ref{th:upper bounds}). The proof is concise to fit mostly in the next few pages.

	\subsection{\label{imed_unimodal notations} Notations}
	Let us consider $\nu \!\in\! \cD$ and let us denote by  $a^\star$ its best arm. We recall that for all $a \!\in\! \cA $,  $\cV_{a} \!=\! \Set{a' \in \cA:\ (a,a') \in E}$ is the neighbourhood of arm $a$ in graph $G\!=\!(\cA,E)$, and that 
	\beq \label{eq:d&epsilon_nu}
	d = \max\limits_{a \in\cA}\abs{\cV_a},\ \epsilon_\nu =  \min\limits_{a  \neq a'}\dfrac{\abs{\mu_a - \mu_{a'}}}{2} \,.
	\eeq
	Then, there exists a function $\alpha_\nu(\cdot)$ such that  for all $0 \!<\!\epsilon\!<\!\epsilon_\nu$, for all $a \neq a^\star $,
	\beq \label{eq:alpha_nu}
	  \KLof{\mu_a \!+\! \epsilon}{ \mu^\star \!-\! \epsilon} \leq \!\left(1 \!+\! \alpha_\nu(\epsilon)\right)^{-1} \KLof{\mu_a}{ \mu^\star}
	\eeq
	and $\lim\limits_{\epsilon\downarrow0}\downarrow\alpha_\nu(\epsilon) = 0$.
	At each time step $t \!\geq\! 1$, $\ahat^\star_t$ is arbitrarily chosen in $\argmin\limits_{a \in \Ahat^\star(t)}N_a(t)$ where $\Ahat^\star(t)\!=\!\argmax\limits_{a \in \cA}\muhat_a(t)$.  
	
	\subsection{Algorithm-based empirical bounds} 
	The \IMEDUB algorithm implies inequalities between the indexes that can be rewritten as inequalities on the numbers of pulls. While  lower bounds involving $\log(t)$ may be expected in view of the asymptotic regret bounds, we show lower bounds on the numbers of pulls involving instead $\log\!\left(N_{a_{t+1}}(t)\right)$, the logarithm of the number of pulls of the current chosen arm. We also provide upper bounds on $N_{a_{t+1}}(t)$ involving $\log(t)$.

	We believe that establishing these empirical lower and upper bounds is a key element of our proof technique, that is of independent interest and not \textit{a priori} restricted to the unimodal structure.
	\begin{lemma}[Empirical lower bounds]\label{lem:unimodal empirical lower bounds}Under \IMEDUB, at each step time $t \!\geq\! \abs{\cA}$, for all $a \!\in\!\cV_{\ahat_t^\star}$,
	\beq \label{eq:lb1}
	\log\!\left(N_{a_{t+1}}(t)\right)  \leq N_{a}(t)\, \KLof{\muhat_{a}(t)}{\muhat^\star(t)} + \log\!\left(N_{a}(t)\right)
	\eeq
	and
	\beq  \label{eq:lb2} 
	N_{a_{t+1}}(t) \leq N_{\ahat^\star_t}(t)\,.
	\eeq
	\end{lemma}
	
	\begin{proof} For $a \!\in\! \cA$, by definition, we have $I_a(t) \!=\! N_{a}(t) \KLof{\muhat_{a}(t)}{\muhat^\star(t)} \!+\! \log\!\left(N_{a}(t)\right) $, hence
		\[
		\log\!\left(N_a(t)\right) \leq I_a(t) \,.
		\]
		This implies, since the arm with minimum index is pulled, $  \log\!\left(N_{a_{t+1}}(t)\right) \!\leq\! I_{a_{t+1}}(t) \!=\! \min\limits_{a' \in \Set{\ahat^\star_t}\!\cup\!\cV_{\ahat^\star_t}} I_{a'}(t) \!\leq\! I_{\ahat^\star_t}(t) \!=\! \log\!\left(N_{\ahat^\star_t}(t)\right)$. By taking the $\log^{-1}(\cdot)$, the last inequality allows us to conclude.
	\end{proof}
	\begin{lemma}[Empirical upper bounds]\label{unimodal empirical upper bounds}
		Under \IMEDUB at each step time $t \!\geq\! \abs{\cA}$,
		\beq \label{eq:ub}	
		N_{a_{t+1}}(t) \,\KLof{\muhat_{a_{t+1}}(t)}{\muhat^\star(t)} \leq \log(t) \,. 
		\eeq
	\end{lemma}
	
	\begin{proof} As above, by construction we have
		\[
		I_{a_{t+1}}(t) \leq I_{\ahat^\star_t}(t) \,.  
		\]
		It remains, to conclude, to note that
		\[
		N_{a_{t+1}}(t) \KLof{\muhat_{a_{t+1}}(t)}{\muhat^\star(t)}  
		\leq I_{a_{t+1}}(t)\,, 
		\]
		and
		\[I_{\ahat^\star_t}(t) =  \log(N_{\ahat^\star_t}(t)) \leq \log(t) \,.
		\]
	\end{proof}
	
	\subsection{Non-reliable current means}
	
	For all arms $a, a'\!\in\!\cA$ and for all accuracy $\epsilon>0$, let $\cE^+_{a,a'}(\epsilon)$ be the set of times where the current mean of arm $a$  $\epsilon$-deviates from above while arm $a$ has  more pulls than the current pulled arm $a'$,
\beq \label{eq:E+}
\cE^+_{a,a'}(\epsilon) \coloneqq \Set{t \in \llbracket 1, T\!-\!1\rrbracket:\ a_{t+1} = a',\ N_{a'}(t) \leq N_{a}(t),\ \muhat_a(t)  \geq \mu_a  + \epsilon } .
\eeq

We similarly define 
\beq \label{eq:E-}
\cE^-_{a,a'}(\epsilon) \coloneqq \Set{t \in \llbracket 1, T\!-\!1\rrbracket:\ a_{t+1} = a',\ N_{a'}(t) \leq N_{a}(t),\ \muhat_a(t)  \leq \mu_a   - \epsilon }  .
\eeq 
We also define 
\beq \label{eq:E}
\cE_{a,a'}(\epsilon) = \cE^+_{a,a'}(\epsilon)\cup\cE^-_{a,a'}(\epsilon)  \,.
\eeq 
\begin{definition}[$\KL$-$\log$ deviation]
For $\epsilon \!>\! 0$, the couple of arms $(a,a') \!\in\! \cA^2$ shows $\epsilon^-$\!-$\KL$-$\log$ deviation at time step $t \!\geq\!1$ if the following conditions are satisfied
\[
\begin{array}{cl}
   (1) &   a_{t+1} = a'  \vspace{1mm} \\
    (2) &  \muhat_a(t) \leq \mu_a - \epsilon \vspace{1mm}\\
    (3) & \log\!\left(N_{a'}(t)\right) \leq   N_a(t)\, \KLof{\muhat_a(t)}{\mu_a\!-\!\epsilon} + \log\!\left(N_{a}(t)\right) .
\end{array}
\]
\end{definition}
For all couple of arms $(a,a')\!\in\!\cA^2$ and for all accuracy $\epsilon>0$, let $\cK^-_{a,a'}(\epsilon)$ be the set of times where  couple of arms $(a,a')$ shows  $\epsilon^-$\!-$\KL$-$\log$ deviation, that is
\beq \label{eq:K-}
\cK^-_{a,a'}(\epsilon) \coloneqq \Set{t \in \llbracket 1, T\!-\!1\rrbracket : 
\begin{array}{cl}
   (1) &   a_{t+1} = a'  \vspace{1mm} \\
    (2) &  \muhat_a(t) \leq \mu_a - \epsilon \vspace{1mm}\\
    (3) & \log\!\left(N_{a'}(t)\right) \leq   N_a(t)\, \KLof{\muhat_a(t)}{\mu_a\!-\!\epsilon} + \log\!\left(N_{a}(t)\right) 
\end{array}
} .
\eeq
We note that
\[
 \cE^{-}_{a,a'}(\epsilon) \subset \cK^-_{a,a'}(\epsilon) \,. 
\]
We can now resort to concentration arguments in order to control the size of these sets, which 
yields the following upper bounds. We defer the proof to Appendix~\ref{app:proof_bouded_sets}.
\begin{lemma}[Bounded subsets of times] \label{lem:bounded_sets} For $\epsilon \!>\!0$, for $(a,a')\!\in\!\cA^2$,
\[
\Esp_\nu\!\left[\abs{\cE^+_{a,a'}(\epsilon)}\right],\ \Esp_\nu\!\left[\abs{\cE^-_{a,a'}(\epsilon)}\right] \leq \dfrac{ 2\sigma_{\!\epsilon}^2\e^{\epsilon^2/2\sigma_{\!\epsilon}^2}}{\epsilon^2} 
\]
\vspace{-1mm}
\[
 \Esp_\nu\!\left[\abs{\cK^-_{a,a'}(\epsilon)\!\setminus\!\cE^-_{a,a'}(\epsilon)}\right] \leq 1 + c_\epsilon^{-1} +  2C_\epsilon  \sqrt{\log(c_\epsilon T)} \,,
\]
where $\sigma_{\!\epsilon}^2 \!=\!\max\limits_{a\in\cA} \Set{\mathbb{V}_{_{X\sim p(\mu')}}(X) \!: \mu'  \!\in\! [\mu_a \!-\!\epsilon\,, \mu_a]}$, $c_\epsilon, C_\epsilon \!>\! 0$ are the constants involved in Theorem~\ref{thm:boundary_crossing}.
\end{lemma}
\subsection{Non-reliable current best arm}
For accuracy $\epsilon>0$, let $\cM^\star(\epsilon)$ be the set of times $t \!\geq\!1$ that do not belong to $\cE^+_{\ahat^\star_t,a_{t+1}}(\epsilon)$ and where some of the current best arm $\ahat^\star_t$ differs from  $a^\star$,
\beq \label{eq:T*}
\cM^\star(\epsilon) \coloneqq \Set{t \geq \abs{\cA}: \begin{array}{cl}
   
    (1)  &  t\notin \cE^+_{\ahat^\star_t,a_{t+1}}(\epsilon) \vspace{2mm}\\ 
    (2)  &  \ahat^\star_t \neq a^\star
\end{array}  } \,.
\eeq

\begin{lemma}[Relation between subsets of times] \label{lem:subset_relations} Under \IMEDUB, for all accuracy $0 \!<\!\epsilon \!<\! \epsilon_\nu = \min\limits_{a\neq a'}\abs{\mu_a \!-\! \mu_{a'}}\!/2$, 
\beq \label{eq:TK}
\cM^\star(\epsilon) \subset \bigcup_{a\in\cV_{\ahat^\star_t}}\cK^-_{a,a_{t+1}}(\epsilon_\nu) \,.
\eeq
\end{lemma}

\begin{proof}  Let us consider $t \!\in\! \cM^\star(\epsilon)$. Since $\ahat^\star_t\!\neq\!a^\star$, there exists $a\!\in\!\cV_{\ahat^\star_t}$ such that 
\beq \label{eq:TK0}
\mu_{a} > \mu_{\ahat^\star} \,.
\eeq
Then, since $\ahat^\star_t\!\in\!\argmax_{a\in\cA}\muhat_a(t)$, we have
\beq \label{eq:TK1}
\muhat_{\ahat^\star}(t) = \muhat^{\star}(t) \geq \muhat_{a}(t) \,. 
\eeq
Since $t \!\in\! \cM^\star(\epsilon)$, $t \!\notin\! \cE^+_{\ahat^\star_t,a_{t+1}}(\epsilon)$. By considering empirical lower bounds \eqref{eq:lb2} and Equation~\eqref{eq:E+}, we have 
\beq \label{eq:TK2}
\mu_{\ahat^\star_t} + \epsilon \geq \muhat_{\ahat^\star_t}(t) \,.
\eeq
By combining Equations~\eqref{eq:TK1}~and~\eqref{eq:TK2}, it comes 
\beq \label{eq:TK3}
\mu_{\ahat^\star_t} + \epsilon \geq \muhat^\star(t) \geq \muhat_{a}(t)\,.
\eeq
Since $\epsilon \!<\! \epsilon_\nu \!\leq\! \abs{\mu_a \!-\!\mu_{\ahat^\star_t}}\!/2$, Equation~\eqref{eq:TK0} and previous Equation~\eqref{eq:TK3} imply
\beq\label{eq:TK4}
\mu_a-\epsilon_\nu >  \muhat_{\ahat^\star_t}(t) \geq \muhat_a(t) \,.
\eeq
Since $a\!\in\!\cV_{\ahat^\star_t}$, empirical lower bounds \eqref{eq:lb1} imply
\beq \label{eq:TK5}
\log\!\left(N_{a_{t+1}}(t)\right)  \leq N_{a}(t)\, \KLof{\muhat_{a}(t)}{\muhat^\star(t)} + \log\!\left(N_{a}(t)\right).
\eeq
The classical monotonic properties of $\KL(\cdot|\cdot)$ and Equation~\eqref{eq:TK4} imply
\beq \label{eq:TK6}
\left\{\begin{array}{ll}
     \muhat_{a}(t) < \mu_a \!-\!\epsilon_\nu  \vspace{1mm}\\
 \KLof{\muhat_{a}(t)}{\muhat^\star(t)} \leq  \KLof{\muhat_{a}(t)}{\mu_a \!-\!\epsilon_\nu}.     
\end{array} \right.
\eeq
Combining Equations~\eqref{eq:TK4}~and~\eqref{eq:TK6}, we get 
\beq \label{eq:TK7}
\left\{\begin{array}{ll}
     \muhat_{a}(t) < \mu_a \!-\!\epsilon_\nu  \vspace{1mm}\\
     \log\!\left(N_{a_{t+1}}(t)\right)  \leq N_{a}(t)\, \KLof{\muhat_{a}(t)}{\mu_a \!-\!\epsilon_\nu} + \log\!\left(N_{a}(t)\right),
\end{array} \right.
\eeq
which means $t\!\in\!\cK^-_{a,a_{t+1}}(\epsilon_\nu)$.
\end{proof}
\subsection{Reliable current means and current best arm}
In this subsection, we characterize subsets of times where both the mean of current pulled arm and the optimal mean are well-estimated. 
	
	Let us consider for $0\!<\!\epsilon\!<\!\epsilon_\nu$, for $a\!\neq\!a^\star$,
	\beq \label{eq:U_a}
	\cU_a(\epsilon) = \Set{t \geq \abs{\cA} :\ a_{t+1} = a} \bigcap \left(\cE^+_{a_{t+1},a_{t+1}}(\epsilon)\cup\cE^-_{\ahat^\star_t,a_{t+1}}(\epsilon)\cup\cE^+_{\ahat^\star_t,a_{t+1}}(\epsilon)\cup\cM^\star(\epsilon)\right).
	\eeq
	Then, Lemma~\ref{lem:subset_relations} implies
	\beq \label{eq:U_a subset}
	\cU_a(\epsilon) \subset  \bigcup_{\substack{a'\in\Set{a}\cup\cV_a \\ a'' \in \cV_{a'} }}\cE^+_{a',a}(\epsilon)\cup\cE^-_{a',a}(\epsilon)\cup\cK^-_{a'',a}(\epsilon_\nu)\,. 
	\eeq
	In particular, from Lemma~\ref{lem:bounded_sets} and previous Equation~\eqref{eq:U_a subset} we have
	\beqa \label{eq:bounded U_a}
	\Esp_\nu\!\left[\cU_a(\epsilon)\right] &\leq& 2 d (d + 1) \dfrac{ 2\sigma_{\!\epsilon}^2\e^{\epsilon^2/2\sigma_{\!\epsilon}^2}}{\epsilon^2}   + d \!\left(\dfrac{ 2\sigma_{\!\epsilon_\nu}^2\e^{\epsilon_\nu^2/2\sigma_{\!\epsilon_\nu}^2}}{\epsilon_\nu^2}  +  1 + c_{\epsilon_\nu}^{-1} +  2C_\epsilon  \sqrt{\log(c_\epsilon T)}\right) \nonumber \vspace{1mm} \\
	&\leq& d(2d+3)  \dfrac{ 2\sigma_{\!\epsilon_\nu}^2\e^{\epsilon_\nu^2/2\sigma_{\!\epsilon}^2}}{\epsilon^2}   + d \!\left(1 + c_{\epsilon_\nu}^{-1} +  2C_\epsilon  \sqrt{\log(c_\epsilon T)}\right),
	\eeqa
	where $d\!=\!\max_{a\in\cA}\abs{\cV_a}$ is the maximum degree of nodes in graph $\cG$.
	\begin{lemma}[Reliable current means] \label{lem:reliable_means} Under \IMEDUB, for all accuracy $0 \!<\!\epsilon \!<\! \epsilon_\nu = \min\limits_{a\neq a'}\abs{\mu_a \!-\! \mu_{a'}}\!/2$, for all sub-optimal arm $a\!\neq\!a^\star$, for all time step $t\!\notin\!\cU_a(\epsilon)$, $ t\!\geq\!\abs{\cA}$,  such that $a_{t+1}\!=\!a$,
	\[
	\left\{ 
\begin{array}{l}
     \ahat^\star_t=a^\star \vspace{1mm} \\
     \muhat^\star(t) \geq \mu^\star -\epsilon \vspace{1mm} \\
     \muhat_a(t) \leq \mu_a + \epsilon \,.
\end{array}
	\right.
	\]
	\end{lemma}

	\subsection{\label{subsec: proof theorem}Upper bounds on the numbers of pulls of sub-optimal arms}
	 In this subsection, we now combine the different results of the previous subsections to prove Theorem~\ref{th:upper bounds}.
	\begin{proof}[Proof of Theorem~\ref{th:upper bounds}.] For $0\!<\!\epsilon\!<\!\epsilon_\nu$, for $a\!\neq\! a^\star$, let us consider $t\!\notin\!\cU_a(\epsilon)$, $t\!\geq\!\abs{\cA}$, such that $a_{t+1}\!=\!a$. From empirical upper bounds \eqref{eq:ub}, we have 
	\beq \label{eq:proof_th_1}
	N_{a}(t) \,\KLof{\muhat_{a}(t)}{\muhat^\star(t)} \leq \log(t) \,.
	\eeq 
	From Lemma~\ref{lem:reliable_means} and Algorithm~\ref{alg:imedub}, we have $a\!\in\!\cV_{a^\star}$ and $ \muhat_a(t)\!\leq\!\mu_a\!+\!\epsilon\!<\!\mu^\star \!-\!\epsilon\!\leq\!\muhat^\star(t)$. From classical monotonic properties of $\KL(\cdot|\cdot)$ and Equation~\eqref{eq:alpha_nu}, we have  $ \KLof{\muhat_a(t)}{\muhat^\star(t)} \!\geq\! \KLof{\mu_a \!+\!\epsilon}{\mu^\star\!-\!\epsilon}\!\geq\!\left(1\!+\!\alpha_\nu(\epsilon)\right)^{-1}\KLof{\mu_a}{\mu^\star}$. In view of Equation~\eqref{eq:proof_th_1}, this implies
	\beq \label{eq:proof_th_2}
	\forall t \notin \cU_a(\epsilon), t \geq\abs{\cA}, \textnormal{ such that } a_{t+1} = a,\quad \left\{
	\begin{array}{l}
	     a \in \cV_{a^\star}  \vspace{1mm}\\
	     N_{a}(t) \leq \dfrac{\left(1+\alpha_\nu(\epsilon)\right) \log(t)}{\KLof{\mu_a}{\mu^\star}} \,.
	\end{array}
	\right.
	\eeq 
	
	\noindent For all arm $a\!\in\!\cA$,  for all time step $t\!\geq\!\abs{\cA}$, we  denote by 
\beq \label{eq:proof_th_3}
\tau_a(t) = \max\Set{t'\in\llbracket\abs{\cA}\,;t\rrbracket:\ a_{t'+1} = a \quad \textnormal{and} \quad t' \notin \cU_{a}(\epsilon) }
\eeq
the last time step  before time step $t$ that does not belong to $\cU_{a}(\epsilon)$ such that we pull arm $a$.\newline

Then, from Equations~\eqref{eq:proof_th_2}~and~\eqref{eq:proof_th_3} we have
\beqan
\forall a \neq a^\star,\ \forall t \geq 1, \quad  N_a(t) &=& N_a\!\left(\abs{\cA}\right) + \sum\limits_{t'\geq\abs{\cA}}^{t - 1} \ind_{\Set{ a_{t'+1} = a }} \nonumber\\
&\leq & 1  + \sum\limits_{t'\geq1}^{t - 1} \ind_{\Set{ a_{t'+1} = a,\ t'\in \cU_a(\epsilon) }} + \sum\limits_{t'\geq\abs{\cA}}^{t - 1} \ind_{\Set{ a_{t'+1} = a,\ t'\notin \cU_a(\epsilon) }} \nonumber \\
&\leq& 1  + \abs{\cU_a(\epsilon)} + \sum\limits_{t'\geq\abs{\cA}}^{t - 1} \ind_{\Set{ a_{t'+1} = a,\ t'\notin \cU_a(\epsilon) }} \\
&\leq& 1  + \abs{\cU_a(\epsilon)} + \ind_{\Set{a \notin\cV_{a^\star} }}\times 0 +  \ind_{\Set{a \in\cV_{a^\star} }}\times N_a\!\left(\tau_a(t)\right)\\
&\leq& 1  + \abs{\cU_a(\epsilon)} +  \ind_{\Set{a \in\cV_{a^\star} }}\dfrac{\left(1+\alpha_\nu(\epsilon)\right) \log\!\left(\tau_a(t)\right)}{\KLof{\mu_a}{\mu^\star}}\\
&\leq& 1  + \abs{\cU_a(\epsilon)} +  \ind_{\Set{a \in\cV_{a^\star} }}\dfrac{\left(1+\alpha_\nu(\epsilon)\right) \log(t)}{\KLof{\mu_a}{\mu^\star}} \,.
\eeqan
This implies
	\beq
	\forall a\neq a^\star, \forall t \geq 1,\quad N_a(t) \leq \left\{
	\begin{array}{ll}
	     \dfrac{\left(1+\alpha_\nu(\epsilon)\right) \log(t)}{\KLof{\mu_a}{\mu^\star}} + \abs{\cU_a(\epsilon)}+1 & \textnormal{if } a \in \cV_{a^\star} \vspace{2mm}\\
	     \abs{\cU_a(\epsilon)}+1 & \textnormal{if } a \notin \cV_{a^\star} \,.
	\end{array}
	\right.
	\eeq
	From Equation~\eqref{eq:bounded U_a}, averaging these inequalities allows us to conclude.
	\end{proof}

	\section{Numerical experiments}
	In this section, we compare empirically the following 
	algorithms : \OSUB, \UTS \citep{combes2014unimodal, kaufmann2020uts} and \IMEDUB described in Algorithm~\ref{alg:imedub}. We illustrate  how performs the \IMEDUB algorithm under Bernoulli, Gaussian (variance $\sigma^2 \!=\!0.25$) or Exponential distribution assumption. For the experiments we consider a graph $\cG$ with  maximal degree $d=2$ and the unimodal unimodal vectors of means $\mu\!=\!(0.05, 0.10, 0.15, 0.20, 0.25, 0.20, 0.15, 0.10, 0.05)$, and average regrets over $500$ runs for each distribution family. Based on these experiments (Figure~\ref{figure111}), it seems that \IMEDUB competes with \OSUB and \UTS.

	\begin{figure}[H] 

  \centering
  \includegraphics[width=0.8\linewidth]{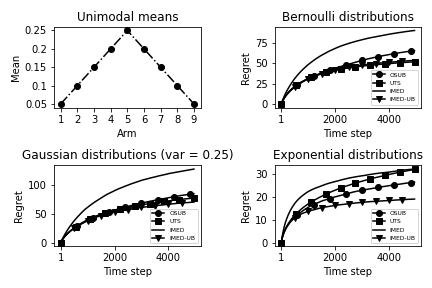}
  \caption{\label{figure111}Cumulative regrets averaged over $500$ runs. }

\end{figure}

	\section*{Conclusion}
	
	In this paper, we have revisited the setup of unimodal multi-armed bandits: We introduced a novel variant based on the \IMED algorithm. This algorithm does not separate exploration from exploitation rounds and is proven optimal for one-dimensional exponential family distributions. Remarkably, the \IMEDUB algorithm  do not require any optimization procedure, which can be interesting for practitioners. We also provided a novel proof algorithm, in which we make explicit empirical lower and upper bounds, before tackling the handling of bad events by specific concentration tools. This proof technique greatly simplifies and shorten the analysis of \IMEDUB.  Last, we provided numerical experiments that show the practical effectiveness of \IMEDUB.

\acks{

This work has been supported by the French Ministry of Higher Education and Research, Inria, the French Agence
Nationale de la Recherche (ANR) under grant ANR-16-CE40-0002 (the BADASS project), the MEL, the
I-Site ULNE regarding project R-PILOTE-19-004-APPRENF.

Pierre Ménard is supported by the SFI Sachsen-Anhalt for the project REBCI ZS/2019/10/102024 by the Investitionsbank SachsenAnhalt.

}


\newpage

\vskip 0.2in
\bibliographystyle{abbrvnat}
\bibliography{biblio}

\newpage

\appendix

\newpage
\section{\IMEDUB finite time analysis}
	\label{app: imed_analysis}
	We regroup in this section, for completeness, the proofs of the remaining lemmas used in the analysis of \IMEDUB in Section~\ref{sec : imed_analysis}.
	
	\subsection{Proof of Lemma~\ref{lem:bounded_sets}} \label{app:proof_bouded_sets}
	
\begin{proof} We start by proving $\Esp_\nu\!\left[\abs{\cE^-_{a,a'}(\epsilon)}\right]\!\leq\!\e^{2\epsilon^2 }/2\epsilon^2$. The proof that $\Esp_\nu\!\left[\abs{\cE^+_{a,a'}(\epsilon)}\right]\!\leq\!\e^{2\epsilon^2 }/2\epsilon^2$ is similar. \newline

\noindent We write 
\beq \label{eq:E_bounded_1}
\abs{\cE^-_{a,a'}(\epsilon)} = \sum_{t =1}^{ T-1}\ind_{\Set{a_{t+1}=a', N_{a'}(t) \leq  N_{a}(t),\ \mu_{a} - \muhat_a(t) \geq \epsilon }} \,.
\eeq 
Considering the stopped stopping times $ \tau_n \!=\! \inf\Set{ t\!\geq\! 1, N_{a'}(t) \!=\! n} $ we will rewrite the sum   of indicators and use Lemma~\ref{lem:time_uniform_concentration_1}.
\beqa \label{eq:E_bounded_2}
\abs{\cE^-_{a,a'}(\epsilon)}&\leq& \sum\limits_{t\geq1}\ind_{\Set{a_{t+1}=a',\  N_{a'}(t) \leq N_{a}(t),\  \mu_{a} - \muhat_a(t)  \geq \epsilon }} \\
&\leq& \sum\limits_{n\geq1}\ind_{\Set{ n-1 \leq N_{a}(\tau_n - 1) ,\ \mu_{a} - \muhat_a(\tau_n - 1)  \geq \epsilon }} \nonumber \\
&\leq& 1 + \sum\limits_{n\geq2}\ind_{\Set{ n-1 \leq N_{a}(\tau_n-1),\ \mu_{a} - \muhat_a(\tau_n - 1)\geq \epsilon }} \,. \nonumber
\eeqa
Taking the expectation of Equation~\eqref{eq:E_bounded_2}, it comes
\beq \label{eq:E_bounded_3} \Esp_\nu\!\left[\abs{\cE^-_{a,a'}(\epsilon)}\right]
\leq  1 + \sum\limits_{n\geq 1} \textbf{\textnormal{P}}_\nu\!\left( \bigcup\limits_{\substack{t \geq 1  \\ N_a(t) \geq n}} \muhat_a(t) \leq \mu_a - \epsilon \right) .
\eeq
From Lemma~\ref{lem:time_uniform_concentration_1}, previous Equation~\eqref{eq:E_bounded_3} implies
\beq \label{eq:E_bounded_4} \Esp_\nu\!\left[\abs{\cE^-_{a,a'}(\epsilon)}\right]
\leq  1 + \sum\limits_{n\geq 1} \exp\!\left(-m\, \KLof{\mu_a\!-\!\epsilon}{\mu_a}\right) .
\eeq
From Lemma~\ref{lem:pinsker}, previous Equation~\eqref{eq:E_bounded_4} implies
\beq \label{eq:E_bounded_11} \Esp_\nu\!\left[\abs{\cE^-_{a,a'}(\epsilon)}\right]
\leq \sum\limits_{n\geq 0} \exp\!\left(-n\epsilon^2/2\sigma_{\!\epsilon}^2\right) = \dfrac{1}{1 - \e^{- \epsilon^2/2\sigma_{\!\epsilon}^2}}\,,
\eeq
where  $\sigma_{\!\epsilon}^2 \!=\!\max\limits_{a\in\cA} \Set{\mathbb{V}_{_{X\sim p(\mu')}}(X) \!: \mu'  \!\in\! [\mu_a \!-\!\epsilon\,, \mu_a]}$.
Finally we note that 
\[
\dfrac{1}{1 - \e^{- \epsilon^2/2\sigma_{\!\epsilon}^2}} = \dfrac{\e^{\epsilon^2/2\sigma_{\!\epsilon}^2}}{\e^{\epsilon^2/2\sigma_{\!\epsilon}^2} - 1} \leq \dfrac{2\sigma_{\!\epsilon}^2\e^{\epsilon^2/2\sigma_{\!\epsilon}^2}}{\epsilon^2}\,.
\]
\newline

We now show  that $\Esp_\nu\!\left[\abs{\cK^-_{a,a'}(\epsilon)}\!\setminus\!\abs{\cE^-_{a,a'}(\epsilon)}\right]\!\leq\!1 \!+\! c_\epsilon^{-1} \!+\!  C_\epsilon  \log\log(c_\epsilon T)$. \newline 

We write 
\beqa \label{eq:K_bounded_1}
&& \abs{\cK^-_{a,a'}(\epsilon)\!\setminus\!\cE^-_{a,a'}(\epsilon)}  \nonumber\\
&=& \!\!\sum_{t =1}^{ T-1}{\ind_{\Set{a_{t+1}=a',\ 1 \leq  N_a(t) < N_{a'}(t),\  \muhat_{a}(t) \leq \mu_{a} - \epsilon,\    \log\left(N_{a'}(t)\right) \leq   N_a(t)\, \KL(\muhat_a(t)|\mu_a\!-\!\epsilon) + \log\left(N_{a}(t)\right) }}} . 
\eeqa
Considering the stopped stopping times $ \tau_n \!=\! \inf\Set{ t\!\geq\! 1, N_{a'}(t) \!=\! n} $ we will rewrite the sum   $ \sum_{t \in \llbracket 1, T\!-\!1\rrbracket}{\ind_{\Set{a_{t+1}=a',\ 1 \leq N_a(t) < N_{a'}(t),\  \muhat_{a}(t) \leq \mu_{a} - \epsilon,\    \log\left(N_{a'}(t)\right) \leq   N_a(t)\, \KL(\muhat_a(t)|\mu_a-\epsilon) + \log\left(N_{a}(t)\right) }}} $ and use boundary crossing probabilities for one-dimensional exponential family distributions.
\beqa \label{eq:K_bounded_2}
&&\abs{\cK^-_{a,a'}(\epsilon)\!\setminus\!\cE^-_{a,a'}(\epsilon)} \nonumber \\
&\leq& \sum_{t =1}^{ T-1}{\ind_{\Set{a_{t+1}=a',\ 1 \leq  N_a(t) < N_{a'}(t),\  \muhat_{a}(t) \leq \mu_{a} - \epsilon,\    \log\left(N_{a'}(t)\right) \leq   N_a(t)\, \KL(\muhat_a(t)|\mu_a-\epsilon) + \log\left(N_{a}(t)\right) }}} \nonumber \\
&=& \sum_{t =1}^{ T-1}\sum_{n =1}^{ T-1} \ind_{\Set{\tau_{n+1} = t+1 }}\ind_{\Set{ 1 \leq  N_a(\tau_{n+1} -1) < n,\  \muhat_{a}(\tau_{n+1}-1) \leq \mu_{a} - \epsilon}}\times\nonumber\\
&&\ind_{\Set{     \log(n) \leq   N_a(\tau_{n+1} -1)\, \KL(\muhat_a(\tau_{n+1} -1)|\mu_a-\epsilon) + \log\left(N_{a}(\tau_{n+1} -1)\right) }} \nonumber\\
&=&  \sum_{n =1}^{ T-1} \ind_{\Set{ 1 \leq  N_a(\tau_{n+1}-1 ) < n,\  \muhat_{a}(\tau_{n+1}) \leq \mu_{a} - \epsilon}}\times\nonumber\\
&&\ind_{\Set{    \log(n) \leq   N_a(\tau_{n+1}-1)\, \KL(\muhat_a(\tau_{n+1}-1)|\mu_a-\epsilon) + \log\left(N_{a}(\tau_{n+1}-1)\right) }} \sum_{t =1}^{ T-1} \ind_{\Set{\tau_{n+1} = t+1 }} \nonumber\\
&\leq&  \sum_{n =1}^{ T-1} \ind_{\Set{ 1 \leq  N_a(\tau_{n+1}-1 ) < n,\  \muhat_{a}(\tau_{n+1}) \leq \mu_{a} - \epsilon,\    \log(n) \leq   N_a(\tau_{n+1}-1)\, \KL(\muhat_a(\tau_{n+1}-1)|\mu_a-\epsilon) + \log\left(N_{a}(\tau_{n+1}-1)\right) }}  \nonumber\\
&=&  \!\!\sum_{n =2}^{ T-1} \ind_{\Set{ 1 \leq  N_a(\tau_{n+1}\!-\!1 ) < n,\  \muhat_{a}(\tau_{n+1}) \leq \mu_{a} - \epsilon,\    \log(n) \leq   N_a(\tau_{n+1}\!-\!1)\, \KL(\muhat_a(\tau_{n+1}\!-\!1)|\mu_a\!-\!\epsilon) + \log\left(N_{a}(\tau_{n+1}\!-\!1)\right) }}\!.
\eeqa

From Equation~\eqref{eq:K_bounded_2}, we get 
\beqa \label{eq:K_bounded_5}
&&\abs{\cK^-_{a,a'}(\epsilon)\!\setminus\!\cE^-_{a,a'}(\epsilon)}  \\
&\leq&  \sum_{n =2}^{ T-1} \ind_{\Set{ 1 \leq  N_a(\tau_{n+1} - 1) < n,\       \KL(\muhat_a(\tau_{n+1}-1)|\mu_a-\epsilon) \geq \log\left(n/N_a(\tau_{n+1}-1)\right) }} . \nonumber
\eeqa
Taking the expectation of Equation~\eqref{eq:K_bounded_5}, it comes
\beqa \label{eq:K_bounded_6}
&&\Esp_\nu\!\left[\abs{\cK^-_{a,a'}(\epsilon)\!\setminus\!\cE^-_{a,a'}(\epsilon)}\right]  \\
&\leq&  \sum_{n =2}^{ T-1} \textbf{\textnormal{P}}_\nu\!\left( \bigcup\limits_{\substack{t \geq 1 \\ \muhat_a(t) < \mu_a-\epsilon  \\ 1\leq N_a(t) \leq n}} \hspace{-4mm} N_a(t) \KLof{\muhat_a(t)}{\mu_a\!-\!\epsilon} \!\geq\! \log\!\left(n/N_a(t)\right)\! \right) . \nonumber
\eeqa
From Theorem~\ref{thm:boundary_crossing}, previous Equation~\eqref{eq:K_bounded_6} implies \beqa \label{eq:K_bounded_7}
&&\Esp_\nu\!\left[\abs{\cK^-_{a,a'}(\epsilon)\!\setminus\!\cE^-_{a,a'}(\epsilon)}\right]  \\
&\leq& 1 +  c_\epsilon^{-1} +  C_\epsilon \sum_{n \geq 1 +  c_\epsilon^{-1}}^{ T-1} \dfrac{c_\epsilon}{c_\epsilon n \sqrt{\log(c_\epsilon n)} } \nonumber \\
&\leq& 1 + c_\epsilon^{-1} +  C_\epsilon \int_{c_\epsilon^{-1}}^T\dfrac{c_\epsilon \,d x}{c_\epsilon x  \sqrt{\log(c_\epsilon x)} }  \nonumber \\
&=& 1 + c_\epsilon^{-1} +  2C_\epsilon  \sqrt{\log(c_\epsilon T)} \,.
\eeqa

\end{proof}

\subsection{Proof of Lemma~\ref{lem:reliable_means}}
\begin{proof} For   $0 \!<\!\epsilon \!<\! \epsilon_\nu = \min\limits_{a\neq a'}\abs{\mu_a \!-\! \mu_{a'}}\!/2$, for  $a\!\neq\!a^\star$, let us consider a  time step $t\!\notin\!\cU_a(\epsilon)$, $t\!\geq\!\abs{\cA}$ such that $a_{t+1}\!=\!a$. \newline

Since $a_{t+1}\!=\!a$ and $t \!\notin\!\cU_{a_{t+1}}(\epsilon)$ then $t \!\notin\!\cE^+_{a_{t+1},a_{t+1}}(\epsilon)$, that is $\muhat_{a_{t+1}}(t) < \mu_{a_{t+1}} \!+\! \epsilon$ or $\muhat_a(t) < \mu_a \!+\! \epsilon$ (since $a_{t+1}\!=\!a$). \newline

Since $a_{t+1}\!=\!a$ and $t \!\notin\!\cU_{a_{t+1}}(\epsilon)$ then $t \!\notin\!\cE^-_{\ahat^\star_t,a_{t+1}}(\epsilon)$, that is 
\beq \label{eq:proof_reliable_means_1}
\muhat^\star(t) = \muhat_{\ahat^\star_t}(t) > \mu_{\ahat^\star_t} -  \epsilon \,.
\eeq

Since $a_{t+1}\!=\!a$ and $t \!\notin\!\cU_{a_{t+1}}(\epsilon)$ then $t \!\notin\!\cE^+_{\ahat^\star_t,a_{t+1}}(\epsilon)\cup\cM^\star(\epsilon)$. From Equation~\eqref{eq:T*}, this implies 
\beq \label{eq:proof_reliable_means_2}
\ahat^\star_t = a^\star \,.
\eeq

By combining Equations~\eqref{eq:proof_reliable_means_1}~and~\eqref{eq:proof_reliable_means_2},  we get
\beq \label{eq:proof_reliable_means_3}
\muhat^\star(t)  > \mu_{a^\star} -  \epsilon = \mu^\star -  \epsilon \,.
\eeq

\end{proof}

	\section{Generic tools} \label{app: concentration_lemmas}
	In this section, Pinsker's inequality for one-dimensional exponential family distributions is reminded. Please refer to Lemma 3 from \cite{CaGaMaMuSt2013} for more insights. We also state two concentration results from \cite{maillard2018boundary}. Relevantly, Theorem~\ref{thm:boundary_crossing} is  the main concentration result used in this paper.	
	
	\begin{lemma}[Pinsker's inequality]\label{lem:pinsker}
For $\mu\!<\!\mu'$, it holds that
\[
\KL(\mu|\mu') \geq  \dfrac{(\mu'-\mu)^2}{2\sigma^2} \,, 
\]
where $\sigma^2 \!=\!\max \Set{\mathbb{V}_{_{X\sim p(\mu'')}}(X) \!: \mu'' \!\in\! [\mu\,, \mu']} $.
\end{lemma}

	\begin{lemma}[Time-uniform concentration] \label{lem:time_uniform_concentration_1} For all arm $a\!\in\!\cA$, for $x\!<\!\mu_a$, $m\!\geq\!1$, we have
\[
\textbf{\textnormal{P}}_\nu\!\left( \bigcup\limits_{\substack{t \geq 1  \\N_a(t) \geq m}} \muhat_a(t) < x \right) \leq \exp\!\left(-m\, \KLof{x}{\mu_a}\right).
\]
\end{lemma}

\begin{theorem}[Boundary crossing probabilities] \label{thm:boundary_crossing} For all arm $a\!\in\!\cA$, for all $\epsilon\!>\!0$, for all $n\!\geq\!1$, we have
\[
\textbf{\textnormal{P}}_\nu\!\left( \bigcup\limits_{\substack{t \geq 1 \\ \muhat_a(t) < \mu_a-\epsilon  \\ 1\leq N_a(t) \leq n}} \hspace{-4mm} N_a(t) \KLof{\muhat_a(t)}{\mu_a\!-\!\epsilon} \!\geq\! \log\!\left(n/N_a(t)\right)\!  \right) \!\leq\! \dfrac{C_\epsilon}{n\sqrt{\log(c_\epsilon n)}} \,,
\]
where $c_\epsilon, C_\epsilon \!>\! 0$  are explained in \cite{maillard2018boundary}.
\end{theorem}

\end{document}